\documentclass{svproc}

\usepackage{times}  \usepackage{helvet}  \usepackage{courier}  \usepackage{url}  \usepackage{graphicx}  \usepackage{dsfont}

\usepackage{xspace}
\usepackage{times}
\usepackage{amsmath,amssymb,amsfonts}
\usepackage{verbatim}
\usepackage{graphicx}
\usepackage{psfrag,graphicx,epsfig,epsf}
\usepackage{graphics}
\usepackage{latexsym}
\usepackage{fancyhdr}
\usepackage{setspace}
\usepackage{color}
\usepackage{url}
\usepackage{mathtools}
\usepackage{amsmath,amsfonts,amssymb,mathrsfs}
\usepackage{psfrag,graphicx,epsfig,epsf}
\usepackage[ruled,linesnumbered, noend]{algorithm2e}
\usepackage{fancyhdr}
\usepackage{wrapfig}
\usepackage{subfigure}
\usepackage{mathtools}
\usepackage{url}
\usepackage{color}
\usepackage{verbatim}
\usepackage{xspace}
\usepackage{enumitem}

\newcommand\numberthis{\addtocounter{equation}{1}\tag{\theequation}}

\title{\Large \bf Pushing the Boundaries of Asymptotic Optimality\\ in Integrated Task and Motion Planning}

\author{Rahul Shome \and Daniel Nakhimovich \and Kostas E. Bekris    }
\institute{Rutgers University, New Brunswick, NJ, USA.}

\newcommand{\graph}{\mathcal{G}}

\newcommand{\fmt}{{\tt FMT}}
\newcommand{\prm}{{\tt PRM}}
\newcommand{\prmstar}{{\tt PRM$^*$}}

\newcommand{\rrt}{{\tt RRT}}

\newenvironment{myitem}{\begin{list}{$\bullet$}
{\setlength{\itemsep}{-0pt}
\setlength{\topsep}{0pt}
\setlength{\labelwidth}{0pt}
\setlength{\leftmargin}{10pt}
\setlength{\parsep}{-0pt}
\setlength{\itemsep}{0pt}
\setlength{\partopsep}{0pt}}}{\end{list}}

\newcommand{\cost}{\mathbf{C}}

\newcommand{\cfree}{\ensuremath{\mathbb{C}_{{\rm free}}}  }

\newcommand{\cobs}{\ensuremath{\mathbb{C}_{{\rm obs}}}  }

\newcommand{\ao}{{\tt AO}}

\newtheorem{assumption}{Assumption}

\newcommand{\workspace}{\mathcal{W}}

\newcommand{\state}{q}
\newcommand{\catstate}{\state}
\newcommand{\startstate}{\state_s}
\newcommand{\goalstate}{\state_g}

\newcounter{model}

\definecolor{darkgreen}{RGB}{30,150,30}

\newcommand{\cspace}{\mathbb{C}}

\newcommand{\orbit}{\mathcal{O}}
\newcommand{\deltaint}{\cspace_{\delta}}

\newcommand{\closedfree}{\overline{\cfree}}

\newcommand{\Chi}{\mathcal{X}}

\newcommand{\ball}[2]{{\mathcal{B}}_{#2}({#1})}
\newcommand{\cone}[4]{{\mathcal{H}}_{#3}({#1},{#2},{#4})}

\newcommand{\prob}{{Pr}}

\newcommand{\trans}{t}
\newcommand{\settrans}{\mathcal{T}}
\newcommand{\mode}{\mathcal{{M}}}
\newcommand{\alg}{\mathtt{ALGO}}
\newcommand{\qstart}{\state_0}
\newcommand{\qtarget}{\state_1}

\newcommand{\oneball}{\mu_1}

\newcommand{\fobt}{\ensuremath{\mathtt{FOBT}}\xspace}
\newcommand{\such}{\text{ such that }}
\newcommand{\iter}{n}

\newcommand{\tamp}{\ensuremath{\mathtt{TAMP}}\xspace}

\newcommand{\edit}[1]{\textcolor{black}{#1}} 

\begin{document}

\maketitle
\thispagestyle{empty}
\pagestyle{plain}

\begin{abstract}
Integrated task and motion planning problems describe a multi-modal state space, which is often abstracted as a set of smooth manifolds that are connected via sets of transitions states. One approach to solving such problems is to sample reachable states in each of the manifolds, while simultaneously sampling transition states. Prior work has shown that in order to achieve asymptotically optimal (AO) solutions for such piecewise-smooth task planning problems, it is sufficient to double the connection radius required for AO sampling-based motion planning. This was shown under the assumption that the transition sets themselves are smooth. The current work builds upon this result and demonstrates that it is sufficient to use the same connection radius as for standard AO motion planning.  Furthermore, the current work studies the case that the transition sets are non-smooth boundary points of the valid state space, which is frequently the case in practice, such as when a gripper grasps an object. This paper generalizes the notion of clearance that is typically assumed in motion and task planning to include such individual, potentially non-smooth transition states. It is shown that asymptotic optimality is retained under this generalized regime.

\end{abstract}

\section{Motivation}
\label{sec:motivation}

\noindent
\begin{wrapfigure}{r}{0.31\textwidth}
    \centering
    \vspace{-0.3in}
    \includegraphics[width=0.3\textwidth]{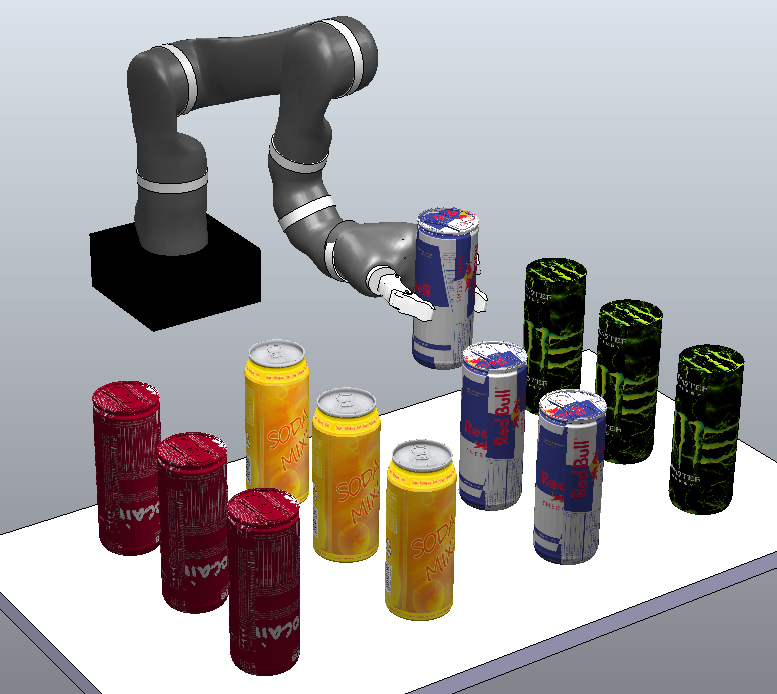}
    \vspace{-0.15in}
    \caption{A multi-modal planning problem where the robot can perform sequences of continuous motions and discrete actions to achieve a target arrangement.}
    \vspace{-0.3in}
    \label{fig:toysims}
\end{wrapfigure}
Integrated task and motion planning (\tamp) corresponds to simultaneously searching for continuous motions and discrete sequences of actions which resolve the target task when combined. Consider the motivational example of prehensile manipulation planning as in Fig.~\ref{fig:toysims}, where a robotic arm must plan both how to move its joints and compute a sequence of discrete grasps and placements. Efficiently solving \tamp problems empowers robots to manipulate objects in the real-world once perception pipelines have provided the location of the detected objects.

Many asymptotically optimal (AO) algorithms for motion planning such as $\prm^*, \rrt^*, \fmt^*$ \cite{Karaman2011Sampling-based-,janson2015fast}, have existed for some time now and high level task planning is typically done with an informed tree search.
Integrated \tamp, however, is more challenging than its two constituent problems as different task primitives, referred to as \textit{task modes}, often correspond to different configuration spaces of possibly varying dimensionality. 
Transitioning between different task modes requires sampling the boundaries between them with a specialized subroutine that is task dependent. For instance, in the case of manipulation planning, sampling the boundary of modes corresponds to identifying different grasps or placements of an object, which, even for simple cases, is not necessarily smooth (see Fig.~\ref{fig:toyproblem} left). This violates assumptions typically made in sampling-based planning literature~\cite{Karaman2011Sampling-based-,janson2015fast,vega2016asymptotically}.

\begin{figure}[t]
  \vspace{-0.4in}
  \begin{center}
            \includegraphics[height=1.1in]{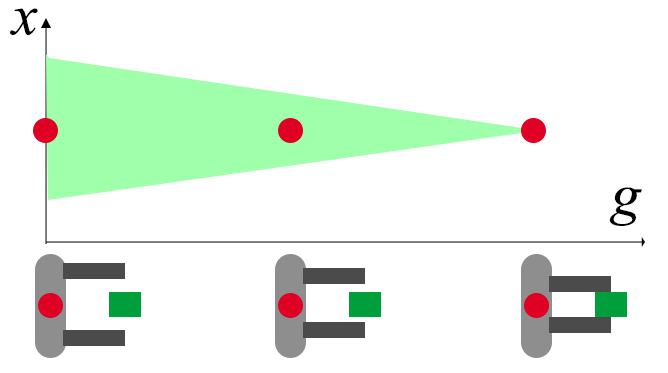}    \hspace{0.3in}
    \includegraphics[height=1.1in]{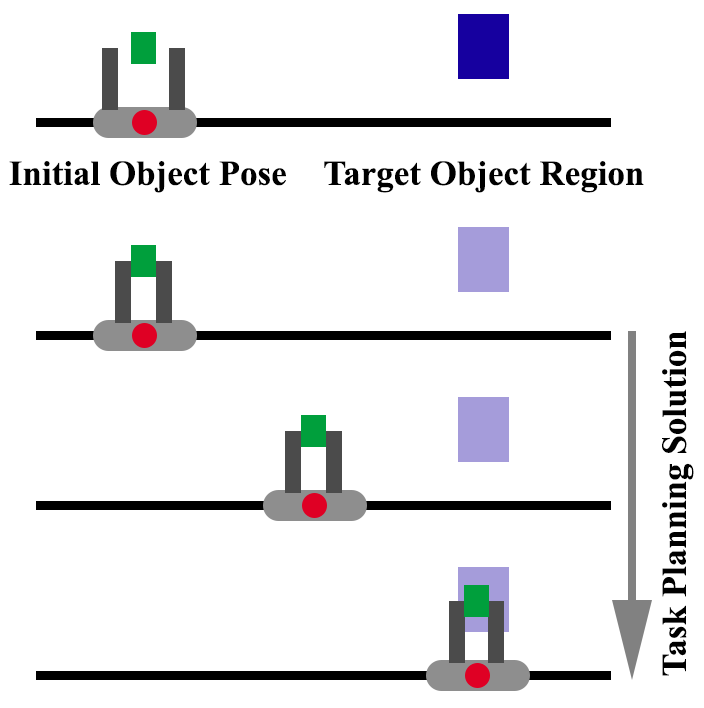}
  \end{center}
    \vspace{-0.3in}
  \caption{(\textit{Left:}) A toy problem with a gripper free to move vertically
    (along $x$), and close its fingers (expressed by the $g$ axis). A
    configuration that touches an object between its fingers lies at
    the non-differentiable boundary, which is the apex of the triangle in the $g$-$x$ graph.
    (\textit{Right:}) A task planning problem involving picking the green object, and placing it in the blue region.
    }
    \vspace{-0.3in}
  \label{fig:toyproblem}
\end{figure}

Early work in manipulation \tamp\ focused on modeling the search space~\cite{alami1990geometrical,koga1994multi}. Recent notable efforts~\cite{hauser2010multi} proposed integrated \tamp\ solutions via multi-modal search strategies, which construct sampling-based roadmaps over individual modes. The key idea constructs an incremental multi-modal \prm\ \cite{Hauser2011Randomized-Multi-Modal-} over a search tree in the space of task modes. \edit{Symbolic planning has also been incorporated into \tamp\ using action predicates~\cite{garrett2018ffrob}, incremental constraints~\cite{dantam2016incremental}, and more recently in an \textit{inside-out} approach by devising a specialized hybrid sampler~\cite{thomason2019unified}.}
The mentioned approaches~\cite{hauser2010multi,Hauser2011Randomized-Multi-Modal-,garrett2018ffrob,thomason2019unified,dantam2016incremental} have been shown to be probabilistically complete, but not necessarily AO.

Other problem instances, such as non-prehensile manipulation~\cite{Dogar:2011ve}, or alternative solution frameworks, such as constrained optimization formulations~\cite{toussaint2015logic}, and hybrid approaches~\cite{havur2014geometric} using answer set programming have also been studied. There are also hierarchical search strategies, which at a low-level call time-budgeted motion planning subroutines that guide the search over viable actions \cite{Kaelbling:2011gb}.
Subsequently, application of such compositional techniques~\cite{schmitt2017optimal,Dobson:2015_MAM,shome2018rearrangement,shome2019anytime}, that perform task planning using sequences of motion plans, were used in domains of manipulation and rearrangement.

\edit{
Of specific interest to the current work is the $\fobt$ algorithm~\cite{vega2016asymptotically}, which was designed for piecewise-analytic task planning domains, using a $\prm^*$-like motion planner. Notably, it argued an \ao\ \tamp solution for a roadmap connection radius \textit{twice} of what was argued for motion planning \cite{Karaman2011Sampling-based-}
A contribution of this work is to demonstrate that previous algorithms~\cite{vega2016asymptotically,Hauser2011Randomized-Multi-Modal-,schmitt2017optimal} can be argued to be \ao\ \textit{without inflating the connection radius}. The key insight is that the \ao\ properties of the motion planning roadmaps can be extended to the \tamp\ domain under a specific set of identified conditions.
}

\edit{
Though not a focus in the current work, the arguments presented here should hold for near-optimal roadmaps~\cite{solovey2018new,solovey2020critical} in the interior of \textit{orbits}, assuming appropriate radius~\cite{solovey2020critical} is chosen for transition connections.
}. 

A standard assumption in the sampling-based planning is the existence of clearance, i.e., minimum distance from obstacles.
This can be violated in task planning, as shown in Fig.~\ref{fig:toyproblem} (left), where a target grasp lies on a non-differentiable boundary point with zero-clearance.
\edit{Another} contribution of the current work is to model what happens in such non-smooth boundaries by extending the definition of robust convergence~\cite{janson2015fast,Karaman2011Sampling-based-}. It is also shown that such countable singular points do not affect the AO properties of integrated task and motion planning.

 
\section{Integrated Task and Motion Planning}
\label{sec:problem_setup}

\edit{This section first defines useful terms and notations, then outlines a general \tamp algorithm $\alg$ that will be analyzed in later sections.}

\noindent{\bf C-space abstraction:} A robotic system is describable by a configuration $\state$ in a $ d $-dimensional
configuration space $ \cspace \subset \mathbb{R}^d $. The robot geometries
exist in a workspace $ \workspace \subset \mathbb{R}^3$, part of which is
occupied by obstacles. This gives rise to an open subset $ \cfree \subset
\cspace $ of the configuration space, which does not result in collisions with
obstacles, and the complement obstacle subset $ \cobs = \cspace / \cfree $. The
boundary of $ \cfree $ is denoted by $ \partial\cfree $, and $ \closedfree $
denotes its closure, i.e., $\cfree$ and its boundary.

\noindent{\bf Paths:} A parameterized continuous curve $ \pi $ in $ \closedfree $ is used to denote
valid paths of the robot as $ \pi:[0,1]\rightarrow\closedfree$. 
A path from $\state_0$ to $\state_1$ in $\closedfree$ is denotes as $\pi_{\state_0\rightarrow \state_1}$
such that $\pi(0) = \state_0$ and $\pi(1) = \state_1$.
Let $ \mathds{P} $ be the set of all valid paths. Then, the path
cost is defined as a mapping $ \cost: \mathds{P} \rightarrow \mathbb{R}^+ $, which
returns a positive measure of a path. The current work considers Euclidean arc-length cost, which is Lipschitz continuous.

\noindent{\bf Modes:} This work adopts the
language of prior work for task planning~\cite{vega2016asymptotically,Hauser2011Randomized-Multi-Modal-} and define a finite set of modes $ \mathfrak{M}
= \{ \mode_0, \mode_1, ... , \mode_k \}$, which correspond to different
operational constraints of the robot for different task components.
The set of all possible configurations within a given mode is denoted as $ \cspace^{\mode_i}$; thus, $ \cspace = \cup_{\mode_i\in\mathfrak{M}} \cspace^{\mode_i}
$. Initially, the discussion will restrict each $ \cspace^{\mode_i} $ to be analytic with smooth boundaries, but this assumption will be later waived in Section 5.

\noindent{\bf Orbits:} Dimensionality reducing constraints force $
\cspace^{\mode_i} \subset \mathbb{R}^{d_{\mode_i}}$ to be a lower dimensional manifold compared to $ \cspace $, i.e., $d_{\mode_i}<d$.
This arises when the task requires the robot to be constrained for some of its degrees of freedom. 
Within these modes, define an orbit $ \orbit_{\mode_i}(x)
$ as a maximal, path-connected, subset of $ \cspace^{\mode_i} $, which contains
configurations $ x \in \cspace^{\mode_i} $. 
Often in manipulation planning, orbits of a mode $\mode_i$ are non-overlapping subsets of a robot's $\cspace^{\mode_i} $ corresponding to a specific grasp for the grasped object and specific placement of the non-grasped objects. 

\noindent{\bf Transitions:} In manipulation planning, the configurations where the robot just grasped or placed an object lie on the border of two modes (two specific orbits of those modes). These configurations at the
intersection of two orbits, are called transition states. Formally, a configuration $ t \in \cfree $ is a transition state if $ t \in
\partial\orbit_{\mode_i} \cap \partial\orbit_{\mode_j} $ for $ i \neq j $.

\noindent{\bf \tamp paths:} For a (non-trivial) task planning problem with a starting configuration of $\startstate$ and goal configuration of $\goalstate$
a feasible solution path will traverse multiple orbits 
over a transition sequence $T = ( \trans_i )_{i=1}^M$ of length $M$.
A feasible \tamp path to a task planning query is denoted as 
$\Pi = \bigoplus_{i=0}^{M} \pi_{\catstate_{i} \rightarrow \catstate_{i+1}},\ \catstate_i \in (\startstate, T, \goalstate)$, 
where $\bigoplus$ denotes path concatenation. The cost $\cost(\Pi)$ of such a path is taken to be the sum
of the costs of paths being concatenated. 
\edit{This formulation implies that traversing a transition has zero (or constant) cost. Such costs are ultimately domain dependent, and cases where the cost function is not locally smooth over transition manifolds are not considered in the current work.
} The optimal path is denoted as $\Pi^*$, and the transition
sequence that it traverses as $T^*$. 
Some additional notation will be used for the analysis of orbital roadmap construction:
Define $ \ball{q}{b} \in \mathbb{R}^d$ as an open $ d
$-dim. hyperball centered at $ q $ with radius $ b $.
Let $ \mu $ denote the measure of the
$\cspace$ space, which corresponds to the Lebesgue measure (generalized
notion of volume). Denote the measure of a ball of radius one as $\oneball$, and
$\iter$ as the number of samples in an orbit.

\subsection{Asymptotic optimality of Sampling-based Algorithms}

This work focuses on algorithms that build roadmaps similar to random geometric graphs ({\tt RGG}) \cite{penrose2003random} via sampling. \edit{It has been shown that sampling-based roadmaps inherit properties of the underlying {\tt RGG}s \cite{solovey2018new}.}

\begin{definition}[Roadmap]
A roadmap is defined as a graph $ \graph_n(V_n, E_n)$, where $ V_n $
corresponds to the $n$ points of a sampling sequence $ \Chi_n$. Edges ${e}(u,v)$
between vertices $u, v \in V_n$ are added in the edge set $E_n$, when:
\begin{myitem}
\item[a)] $\|u,v\| \leq r_n,$ where $ r_n $ is the connection radius of the
	roadmap, and
\item[b)] if the geodesic path connecting $u$ to $v$ is collision free.
\end{myitem}
\label{def:roadmap}
\end{definition}

Traditional roadmap construction~\cite{Kavraki1996Probabilistic-R,Karaman2011Sampling-based-} has focused on the interior of $\cfree$. Given any start $\state_0$ and goal configuration $\state_1$ in the interior of $\cfree$, asymptotic optimality is defined in terms of the optimal path ($\pi^*$) connecting $\state_0 $ to $\state_1$.

\begin{definition}[Asymptotically Optimal Motion Planning on $ \graph_{n} $] An algorithm that builds a roadmap $ \graph_n $ in $\cfree$, and returns the shortest path $
\pi_{\graph_n} $ connecting a start and goal query point lying in the interior of $\cfree$, is asymptotically
optimal \cite{janson2015fast} if:
$\cost(\pi_{\graph_{n}}) \leq (1+\epsilon) \cdot \cost(\pi^*) \
\ \ \forall\ \epsilon>0, as\  n\rightarrow\infty.$

\label{def:motion_ao}
\end{definition}
Prior work~\cite{Karaman2011Sampling-based-,janson2015fast} has provided precise bounds on the radius $r_n$ so that the roadmap $\graph_n$ satisfies this property for all start and goal query points in the interior of $\cfree$. 

\vspace{-0.12in}
\subsection{Algorithmic Outline: Forward Search Tree over Orbital Roadmaps}

Consider a high-level task planning algorithm $ \alg $ that
maintains roadmaps on orbits as well as a high level \textit{forward search tree} representing the
connectivity of orbits through transition states. This could be done explicitly
or implicitly through factorization \cite{Hauser2011Randomized-Multi-Modal-,vega2016asymptotically,garrett2017sample}.
Fig.~\ref{fig:mode_space} depicts such an abstraction of the task planning space with roadmaps constructed inside orbits.
Consider the multi-modal problem where a
robot is tasked to pick-and-place a single object . \edit{The planner begins in $\startstate$ in $\orbit_0$ where $\mode_0$ corresponds to \textit{transit}. Transitions to the adjacent \textit{transfer} mode $\mode_1$ are achieved through sampled grasps to reach $\orbit_1,\orbit_2$. Samples in the interior of $\orbit_0$ connect $\startstate$ to $\trans_1, \trans_2$.   Then, sampled stable positions with the object reach $\orbit_3, \orbit_4$. $\trans_3$ is an arm configuration that achieves the desired object placement with grasp $\trans_1$. When $\goalstate$ is reached in $\orbit_4$ with motions connecting  $\startstate\rightarrow\trans_1\rightarrow\trans_3\rightarrow\goalstate$ the solution can be reported. 
}

\begin{figure}[h]
\vspace{-0.25in}
\centering
\includegraphics[height=1.15in]{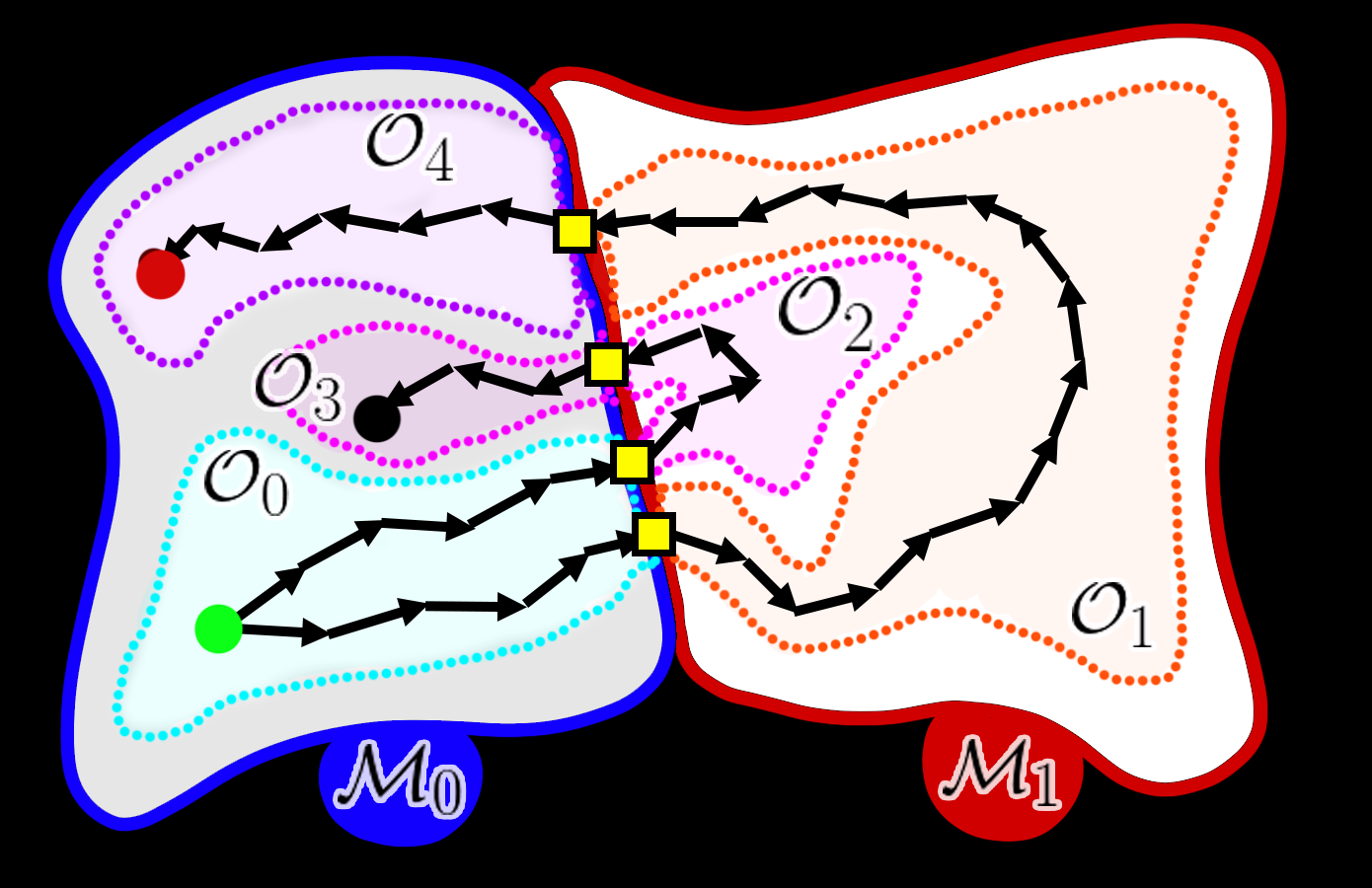}
\includegraphics[height=1.15in]{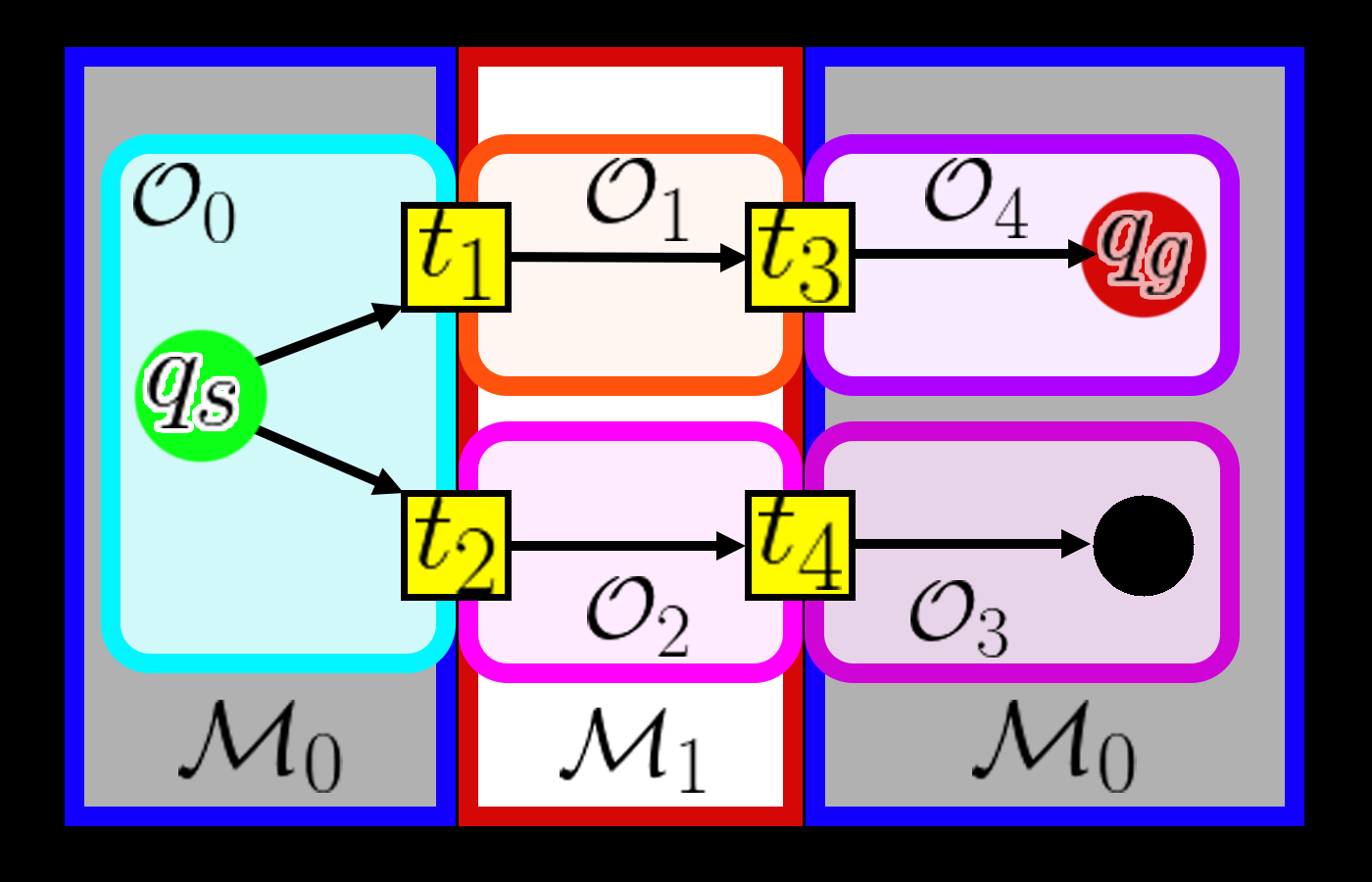}
\vspace{-0.15in}
\caption {
\textit{Left:} A $\cspace$-space split into 2 modes with roadmaps (black) drawn
within orbits connected by transition states (pink). The start and goal
configurations are drawn in green and red respectively.
\textit{Right:} A high-level orbital graph keeping track of the connections
between the roadmaps within orbits (nodes labeled $O_i$) through the transition
states (nodes labeled $t_j$). 
}
\vspace{-0.2in}
\label{fig:mode_space}
\end{figure}

\begin{definition}[Asymptotic Optimality of $\alg$] Algorithm $\alg$, which returns a feasible path $\Pi^\iter$ after $\iter$ iterations, is asymptotically optimal if:
$\cost(\Pi^\iter) \leq (1+\epsilon) \cdot \cost(\Pi^*) \quad \forall \epsilon>0,\text{ as }\iter\rightarrow\infty.$

\label{def:ao}
\end{definition}
\noindent The conditions for $ \alg $ to achieve AO are the following: (i) It can sample transition states over positive measure subsets of mode boundaries; (ii) The orbital roadmap eventually connects every configuration inside an orbit explored by $ \alg $; (iii) The number of sampled transitions $n_t\rightarrow\infty$ as the number of iterations $ \iter\rightarrow\infty $; (iv) For each discovered orbit $\orbit_i$ the size of its roadmap $n_i\rightarrow\infty$ as $ \iter\rightarrow\infty $.

Algorithm~\ref{algo:est} provides a high-level description of $\alg$.
The inputs are the initial configuration $\startstate$, a goal region $G$, and two positive parameters: the number of interior orbit samples $N_m$ and neighboring transitions samples $N_t$ added to the roadmap per iteration from the expanded orbit. A \textit{forward search-tree} $\mathds{T}$ over orbits is built, beginning with the orbit containing $\startstate$. Each iteration, every orbit must have a non-zero probability of being selected (assured by random node selectiion) as $\orbit_{\rm sel}$. The subroutine $\mathtt{expand\_roadmap}$ adds more nodes and edges to an asymptotically optimal roadmap construction algorithm (e.g., \prmstar, or {\tt FMT}$^*$) with $N_m$ more samples. The orbits are not known apriori but $N_t$ new transition points are uniformly selected from the boundaries with all neighboring modes, and connected to the roadmap $\graph_{\orbit_{\rm sel}}$. 
Each of the new (empty) orbits are added to the tree $\mathds{T}$, so that they get a chance to be expanded in the future. $\Pi$ keeps track of the best cost path that reaches the goal region $G$.

\vspace{-0.3in}
\begin{algorithm}[ht]
\caption{{\tt }$\alg (\startstate, G, N_m, N_t $)}
\label{algo:est}
$\Pi\leftarrow\emptyset$;\\
$\mathds{T}(V,E)$;\\
$\mathds{T}.V \leftarrow \mathds{T}.V \cup \orbit(\startstate)$;\\

\For{$\iter$ iters}
{
$\orbit_{\rm sel}\leftarrow \mathtt{uniform\_random}(\mathds{T}.V)$;\\
$\graph_{\orbit_{\rm sel}} \leftarrow \mathtt{expand\_roadmap}(\graph_{\orbit_{\rm sel}}, N_m)   $;\\
$\mathcal{N}_{\orbit_{\rm sel}} \leftarrow \mathtt{uniform\_boundary\_sample}(\orbit_{\rm sel}, N_t)   $;\\
$\graph_{\orbit_{\rm sel}} \leftarrow \mathtt{add\_transitions}(\graph_{\orbit_{\rm sel}}, \mathcal{N}_{\orbit_{\rm sel}})   $;\\
$\mathds{T}.V \leftarrow \mathds{T}.V \cup \mathcal{N}_{\orbit_{\rm sel}}$;\\
$\mathds{T}.E \leftarrow \mathds{T}.E \cup \{ (\orbit_{\rm sel}, \orbit_{\rm neighbor}) \forall \orbit_{\rm neighbor} \in \mathcal{N}_{\orbit_{\rm sel}} \}$;\\
$\Pi \leftarrow \mathtt{retrace\_path}(G)$;\\
}

\textbf{return} $\Pi$;
\end{algorithm}
\vspace{-0.3in}
  
\edit{$\alg$ is  introduced as a general version of} several existing algorithms \edit{that are similar in structure, while differing in the order and exact nature of exploration of the \tamp\ search space. As such, arguments made in Sec.~\ref{sec:ao_thms} should apply to algorithms that maintain the key properties of $\alg$.}
The original Multi-modal PRM-based algorithm~\cite{hauser2010multi} expands every orbit per iteration. So does more recent work~\cite{schmitt2017optimal}. Random-MMP~\cite{Hauser2011Randomized-Multi-Modal-} selects one orbit and samples one neighboring transition per iteration.
More recently $\fobt$~\cite{vega2016asymptotically} samples an orbit per iteration, and if it hasn't been previously explored, sets $N_m = {\Theta}(N), N_t = {\Theta}(N)$ and argues asymptotic optimality when $N\rightarrow \infty$. The arguments made for arguing the AO of $\fobt$ will be summarized in the next section.

\section{Summary of previous results: \fobt}\vspace{-0.1in}

\edit{This section summarizes a previous result showing that \fobt is AO to highlight key assumptions for the new result presented later.}

Previous work proposed the \fobt algorithm~\cite{vega2016asymptotically}. To show asymptotic optimality, this work uses topological tools to build a geometric construction, which traces a robust task planning solution possessing clearance in the interior of each orbit, while maintaining bounded error from the optimal task planning path. 
\fobt uses the following rule for the connection radius of the roadmap in each orbit $\orbit$:
$$\quad  r_{\orbit}(n) >\ \ 4 (1+\frac{1}{d_{\orbit}})^{\frac{1}{d_{\orbit}}} \Big(\frac{\mu(\orbit)}{\oneball}\Big)^{\frac{1}{d_\orbit}} \Big( \frac{\log n}{n} \Big)^{\frac{1}{d_\orbit}},~\refstepcounter{equation}\quad \quad \quad(\theequation) \label{eq:fobtconnection}$$

\noindent where $d_{\orbit}$ is the dimensionality of the orbit,
$\mu$ being the Lebesgue measure or volume, and $\oneball$ is the volume of a unit hyperball. 
\textit{Note that this radius is effectively \textbf{twice} the radius sufficient for \prmstar~\cite{Karaman2011Sampling-based-} for AO motion planning within an orbit.}

\begin{theorem}[\cite{vega2016asymptotically} Theorem 1]
Let $\graph_\iter$ be a geometric graph with $\iter$ vertices constructed using the connection radius in Eq.~\ref{eq:fobtconnection} across the orbits of a task planning space. Let $\cost_\iter^\fobt$ be the minimum cost of a path on $\graph_\iter$ and $\cost^*$ the optimal solution cost, then:
$
    \prob ( \{ \underset{\iter\rightarrow \infty}{\limsup}\quad \cost_\iter^\fobt = \cost^* \} ) = 1.
$
\label{eq:event}
\end{theorem}
\vspace{-0.1in}
Consider segments of the optimal path $\pi^*_{\trans_i\rightarrow \trans_{i+1}}$, which traverses an orbit between two transition states (Fig~\ref{fig:orbital_connection}) that are a concatenation of: \\
(i) One hyperball whose closure contains the start of the segment. (Assuming smooth boundaries, say $\ball{\state_{t^-}}{\delta}, \such \trans_i \in \overline{\ball{\state_{t^-}}{\delta}}$);\\
(ii) One hyperball whose closure contains the end of the segment. (Assuming smooth boundaries, say $\ball{\state_{t^+}}{\delta}, \such \trans_{i+1} \in \overline{\ball{\state_{t^+}}{\delta}}$);\\
(iii) A strongly $\delta$-clear path (refer to it as: $\pi_{\state_{t^-} \rightarrow \state_{t^+}}$) joining the two hyperballs through the interior of the space, i.e., one that is always some $\delta>0$ away from obstacles.

\begin{assumption}[Construction of Segments]
Similar to motion planning setups~\cite{Karaman2011Sampling-based-,janson2015fast}, the interior of each segment $\pi^*$ is robustly optimal, i.e., there exists a sequence of strongly clear paths, which are homotopically equivalent to optimal segments.
\label{ass:robustsegment}
\end{assumption}

\begin{assumption}[Hyperballs Around Transitions]
The boundaries containing between orbits in neighboring modes
are assumed to be smooth (lower dimensional) manifolds, such that it is possible to describe small enough hyperball regions of radius independent of the algorithm. It is additionally assumed that a \textit{transition sampler} is capable of discovering all such positive volume regions by uniformly sampling the appropriate sub-manifold.
\label{ass:boundaryballs}
\end{assumption}

Restating Equation 20 from \cite{vega2016asymptotically}, the probability of failing to be asymptotically optimal is shown to diminish as follows for any $\epsilon > 0$
\begin{align*}
    &\prob(\cost_\iter^\fobt \geq (1+\epsilon)\cost^*) \leq \\
    &\prob( \{ \text{Failing to sample in the neighborhood of transitions after $n$ samples} \}  ) + \\
    &\prob( \{ \text{Failing to trace bounded error path inside orbit after n samples} \}  ) \numberthis
\label{eq:failure_prob}
\end{align*}

It was shown that both of the probability terms on the right-hand side of the inequality go to $0$ as $\iter\rightarrow\infty$. Given the query points $\startstate$ and $\goalstate$, the output from the algorithm ($\Pi^\iter_\fobt$) after $\iter$ iterations can be described as follows.

{\bf Output of \fobt($\startstate,\goalstate$)}:
\begin{align}\vspace{-0.1in}
    \Pi^\iter_\fobt = \bigoplus_{i=0}^{M} \pi_{\catstate_{i} \rightarrow \catstate_{i+1}} \such \cost^\iter_\fobt = \cost(\Pi^\iter_\fobt) \leq (1+\epsilon)\cost^*\\
    \text{where } \catstate_i \in (\startstate, T^\fobt, \goalstate), \quad T^\fobt = (\trans_i)_{i=1}^M.
    \label{eq:existence}
\end{align}
A chief takeaway from the \fobt algorithm is that discovering a $\trans_i$ is guaranteed by uniformly sampling the boundaries of modes (or orbits). Note that this property is a consequence of the assumptions about the spaces involved, and not a feature of the algorithm \fobt. It suffices to give this boundary sampling subroutine enough attempts, which tend to infinity. 
\section{New arguments in integrated task and motion planning}
\label{sec:ao_thms}

Previous work has provided an analysis that sampling-based planning solutions for \tamp are AO as long as they use a connection radius twice as large as that of AO sampling-based motion planners~\cite{vega2016asymptotically}. 
\edit{This section builds a sequence of arguments to}
show that the same connection radius as that in motion planning is also sufficient for \tamp. 

{\bf Summary of Arguments:} \edit{Theorem 2 demonstrates the existence of a ``robust'' solution to the \tamp problem given the current assumptions. Theorem 3 argues that an optimal \tamp solution must comprise of a sequence of optimal paths between orbit transitions. Theorem 4 argues that an AO roadmap-based motion planner will remain AO even for a query between a start and a goal on the boundary of an orbit. Theorem 5 uses Theorems 3 and 4 to argue that $\alg$ is AO if it can find a robust transition sequence. Theorem 6 details the conditions necessary for a boundary sampler to find such a transition sequence. Finally, bringing it all together, Theorem 7 proves that $\alg$ is AO.}

Define the set of all possible valid finite transition sequences
\begin{align*}
	\settrans = \{ T=(\trans_i)_{i=1}^{M}\ | \ 0<M\leq M_{max} \},
\end{align*}
where each $ \trans_i $ is a transition state, and motion planning in a single orbit can connect to an orbit in a neighboring mode through $ \trans_{i+1} $. $M_{max}$ is assumed to be finite as in previous work~\cite{hauser2010multi,vega2016asymptotically}.  

With a slight abuse of notation, define the cost of this transition sequence $\cost(T)$ as the least-cost task planning solution that can be obtained over $T$ between an implicit start and goal state, traced along piecewise robustly optimal segments (similar to Assumption~\ref{ass:robustsegment}).

\begin{theorem}[Robust Optimality of Task Planning]
For a task planning query with piece-wise analytic constraints, there exists a sequence of hyperballs on the boundaries with radius independent of $n$, such that any transition sequence passing through these regions has bounded error to the optimal cost. There can be multiple such sequences.

More specifically, there exists a non-empty set of transition sequences $ \settrans^+ \subset \settrans$ such that for all small $ \epsilon^+>0 $
$$\cost(T^+)\leq(1+\epsilon^+)c^*, \ \ \forall T^+\in\settrans^+, $$

\edit{
Without loss of generality, set $T^+ = (\trans_i)_{i=1}^{M}$. Then, there exist $M$ balls $\ball{\trans_i}{\theta}$ centered around transitions $\trans_i$ with radius $\theta > 0$ such that for all sequences $T' = (\trans_i')_{i=1}^{M}$ with $\trans_i' \in  \ball{\trans_i}{\theta}$ the cost is similarly bounded as 
$\cost(T')\leq(1+\epsilon^+)c^* .$
}

Note that each hyperball $\ball{\trans_i}{\theta}$ has a dimensionality identical to the submanifold where $\trans_i$ was sampled. 
\label{thm:robust_task}
\end{theorem}

\begin{proof}
The proof is a direct consequence of Assumption~\ref{ass:boundaryballs}, and the guaranteed existence of $T^\fobt$ as shown in Eq.~\ref{eq:existence}, which means for 
$$ \settrans^+ = \{T^\fobt \}, \quad \epsilon^+ = \epsilon\ \text{(from \fobt)} $$
the existence of $\settrans^+$ is assured. This directly proves the robust optimality of $\Pi^*$ in the described setup for task planning. This additionally guarantees that the property holds for all small $\epsilon^+>0$.\qed
\end{proof}

\begin{wrapfigure}{r}{0.5\textwidth}
	\centering
 	\vspace{-0.37in}
    \includegraphics[width=0.49\textwidth]{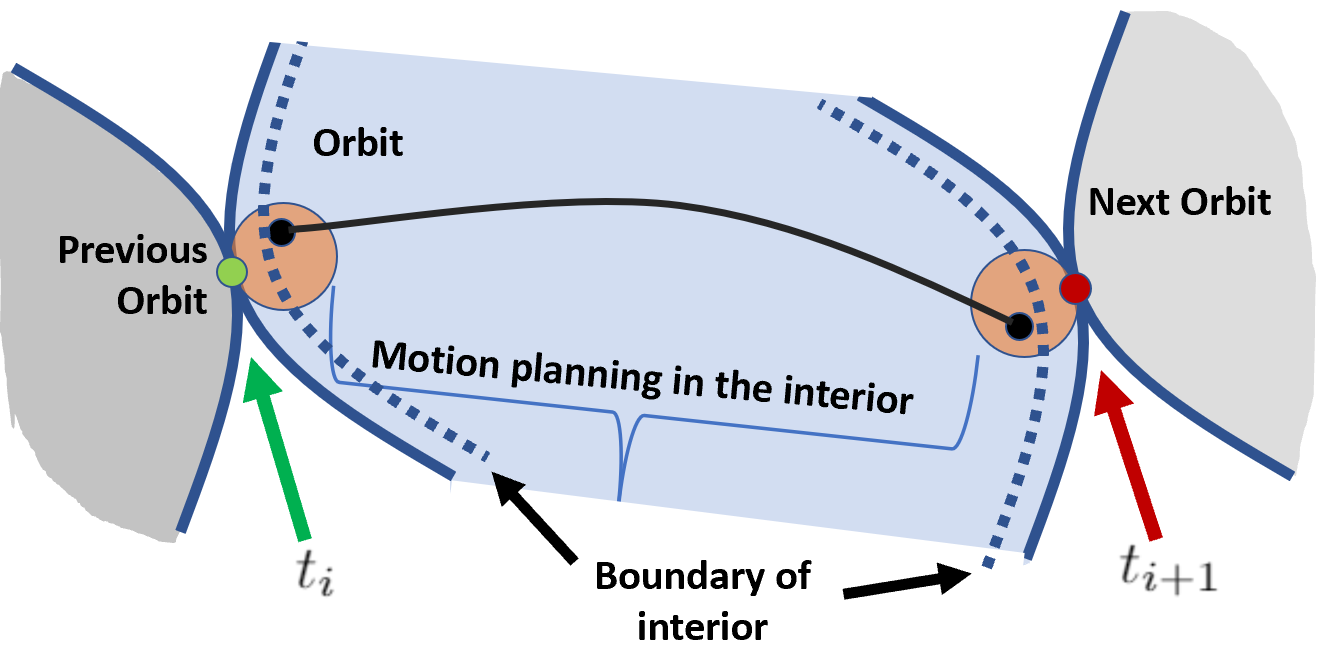}
    \vspace{-0.2in}
    \caption{The image shows the segment of the trajectory inside an orbit, where the problem is decomposed into connecting to the transition points and motion planning in the interior.}
    \vspace{-0.3in}
    \label{fig:orbital_connection}
\end{wrapfigure}
\noindent\textbf{Note:} This is a straightforward extension of the arguments presented in the previous work, and such robustness was inherently assumed therein. 
It should be pointed out that even though the minimal set $\settrans^+ = \{T^\fobt \}$ which essentially lies in the neighborhood of $T^*$ suffices to argue robust optimality, the relaxation of the definition to allow the existence of an arbitrary set $\settrans^+$ captures a lot of situations in task planning where the optimal solution is rarely unique. Consider the problem of rearranging objects $A$ and $B$. It is possible, the two solutions that transfer $A$, then $B$ versus transferring $B$ then $A$, are effectively identical in cost, though drastically different in terms of how the transition sequences look.

\noindent\textbf{Implication:} Theorem~\ref{thm:robust_task} indicates that there are positive volumes (the hyperballs around \textit{desired} transition configurations) in the submanifold on which these transitions exist, and are sampled. This allows the cost of discovered solutions connected through these to have (an arbitrarily small) bounded deviation from the optimal cost.

\begin{theorem}[Pairwise-optimal Planning Over Robust Transition Sequence]
Given an $\epsilon^+$-robust transition sequence $T^+$, a path $\Pi$ constructed from optimal orbital segments traversing $T^+$ maintains the $\epsilon^+$ cost bound.
\end{theorem}
\begin{proof}
Let $T^+ \in \settrans^+$ and let $(\startstate, T^+, \goalstate)$ be the sequence of transitions with start and goal configurations concatenated at either end. Let $\pi_{\catstate_i\rightarrow \catstate_{i+1}}$ denote a feasible motion planning solution over a pairwise connection and 
$\pi_{\catstate_i\rightarrow \catstate_{i+1}}^*$ be an optimal connection. 

Let $\Pi^+$ be the path concatenations of $\pi_{\catstate_i\rightarrow \catstate_{i+1}}$ such that $\cost(\Pi^+) \leq \cost(T^+)$. Let $\Pi^*$ be the path concatenations of $\pi_{\catstate_i\rightarrow \catstate_{i+1}}^*$.
The result immediately follows since
\begin{align*}
  \cost( \pi_{\catstate_i\rightarrow \catstate_{i+1}}^*) \leq \cost(\pi_{\catstate_i\rightarrow \catstate_{i+1}}) \quad& \forall i,\ 0 \leq i \leq M\\
  \implies
   \sum_{i=0}^M\ \cost(\pi_{\catstate_i\rightarrow \catstate_{i+1}}^*)
  \leq&\ \cost(T^+) \leq (1+\epsilon^+)c^*
\end{align*}
\textit{This implies that it suffices to reason about the optimality of the pairwise motion planning problems, as long as the set of transitions from $\settrans^+$ are sampled.}\qed
\end{proof}

\begin{theorem}[Pairwise-asymptotically Optimal Planning Between Robust Transition Sequence]
Given $\graph_n$ constructed by an asymptotically optimal roadmap-based planner and solution path $\pi_{\trans_i\rightarrow \trans_{i+1}}^n$ found from $\graph_n$ for two transitions $\trans_i$ to $\trans_{i+1}$, then 
$$ \lim_{n \rightarrow \infty}  \prob( \{ \cost(\pi_{\trans_i\rightarrow \trans_{i+1}}^n) \geq (1+\epsilon)\cost(\pi_{\trans_i\rightarrow \trans_{i+1}}^*) \} ) = 0,\quad \forall \epsilon >0 $$
\label{thm:pairwiserobust}
\vspace{-0.3in}
\end{theorem}

\begin{wrapfigure}{r}{0.32\textwidth}
    \centering
    \vspace{-0.02in}
    \includegraphics[width=0.31\textwidth]{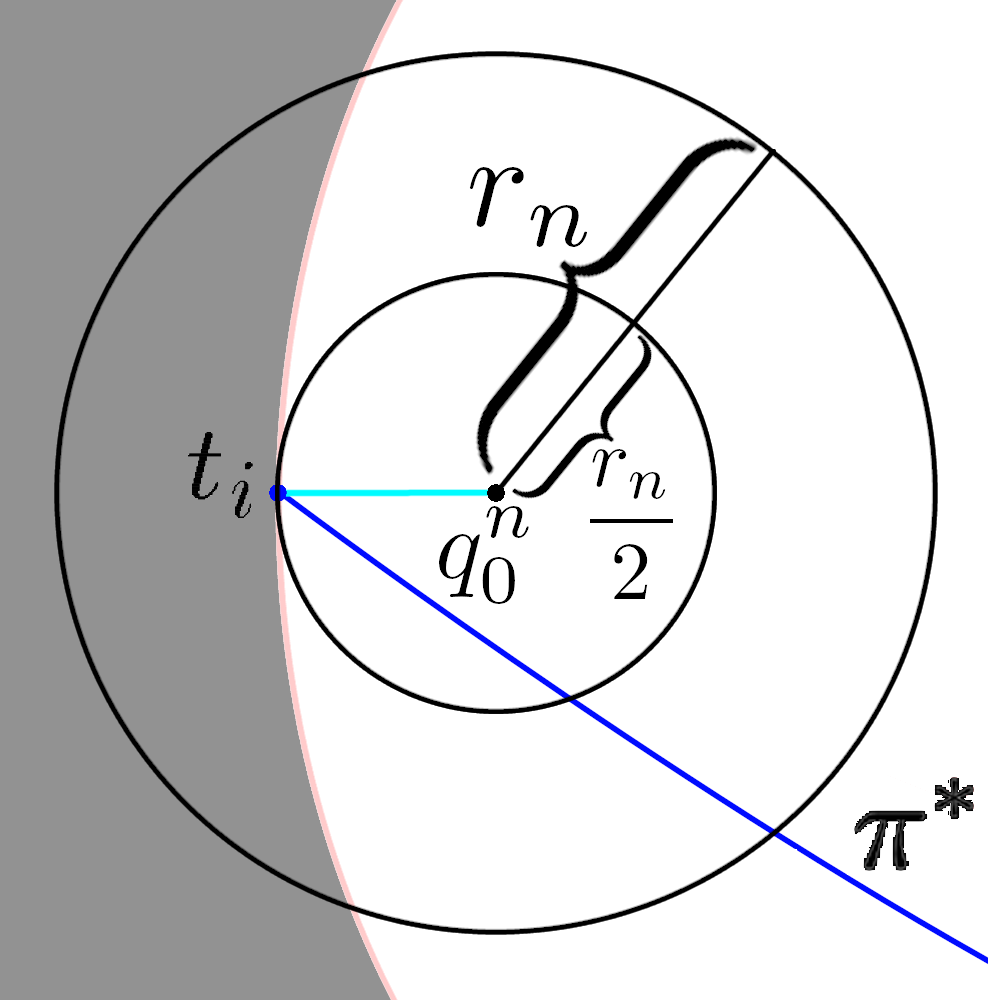}
    \vspace{-0.15in}
    \caption{Construction at smooth transition boundary.
        }
    \vspace{-0.3in}
    \label{fig:smooth}
\end{wrapfigure}
$\empty$

\begin{proof}
This needs some additional consideration since the transition configurations lie on the boundary of the orbit, instead of the interior. Traditional sampling-based roadmap methods (\cite{Karaman2011Sampling-based-,janson2015fast}) guarantee the following event for interior points, say from $\state_0$ to $\state_1$. Restating Def~\ref{def:motion_ao}:
$ \lim_{n \rightarrow \infty}  \prob( \{ \cost(\pi_{\state_0\rightarrow \state_1}^n) \geq (1+\epsilon)\cost(\pi_{\state_0\rightarrow \state_1}^*) \} ) = 0,\quad \forall \epsilon >0. $

Note that $\state_0$, and $\state_1$ are not unique in any way, and the above property essentially holds for any two points in the interior of the space.
Let the connection radius be $r_n>0$.  Given Assumption~\ref{ass:robustsegment} about the segment connecting $\trans_i$ to $\trans_{i+1}$, there is a small enough hyperball that can \textit{touch} the smooth boundary points (Fig.~\ref{fig:smooth}).
Construct balls of radius $\frac{r_n}{2}$. As $n$ increases, $r_n$ decreases and at some point becomes sufficiently small to satisfy Assumption~\ref{ass:robustsegment}. Any point in the interior of such a ball is in the \textit{interior} of the space. Let $\state_0^n$ be the center of such a ball, that gets closer to $\trans_i$ as $n$ grows and  $\frac{r_n}{2}$ shrinks. Any sample in such a ball must be connected to $\trans_i$.

Inspect the event that such a hyperball of radius $\frac{r_n}{2}$ fails to have a sample in sample set $\Chi_n$ for either $\trans_i$ or $\trans_{i+1}$  :
$\{ \ball{\state_0^n}{\frac{r}{2}} \cap \Chi_n = \emptyset  \}$ and $\{ \ball{\state_1^n}{\frac{r}{2}} \cap \Chi_n = \emptyset  \}$.
\begingroup
\allowdisplaybreaks
\begin{align}
    \prob(\{ \ball{\state_0^n}{\frac{r}{2}} \cap \Chi_n = \emptyset  \}) &= \Big(  1 - \frac{\mu(\ball{\state_0^n}{\frac{r}{2}})}{\mu(\orbit)} \Big)^n\label{eq:connect_bound0}\\
    &= \Big( 1 - \frac{\mu_1}{{\mu(\orbit)}} \Big(\frac{r_n}{2}\Big)^d \Big)^n \leq e^{-  \frac{\oneball}{2^d\mu(\orbit)} n r_n^d }\label{eq:connect_bound1}\\
    \implies \lim_{n\rightarrow\infty} \prob(\{ \ball{\state_0^n}{\frac{r}{2}} \cap \Chi_n = \emptyset  \}) &= 0 ,\quad \text{when}\ \ \lim_{n\rightarrow\infty}n r_n^d \rightarrow \infty
    \label{eq:connect_bound}
\end{align}
\endgroup
The same argument holds for $\trans_{i+1}$. It is evident that the connection radii recommended by sampling-based roadmap planners already make this probability go to 0.

Let $\state_0$ be \textit{the} sample that the event $\{ \ball{\state_0^n}{\frac{r}{2}} \cap \Chi_n \neq \emptyset  \}$ generated, and similarly $\state_1$ for $\trans_{i+1}$. The exact configurations do not matter as they are guaranteed to be in the interior of the space.

So combining the failure conditions and by using the union bound it is possible to write:
\begin{align}
\prob( \{ \cost(\pi_{\trans_i\rightarrow \trans_{i+1}}^n) \geq& (1+\epsilon)\cost(\pi_{\trans_i\rightarrow \trans_{i+1}}^*) \} )\\ \leq
\ &\prob(\{ \ball{\state_0^n}{\frac{r}{2}} \cap \Chi_n = \emptyset  \})\\
+ &\prob( \{ \cost(\pi_{\state_0\rightarrow \state_1}^n) \geq (1+\epsilon)\cost(\pi_{\state_0\rightarrow \state_1}^*) \} )\\
+ &\prob(\{ \ball{\state_1^n}{\frac{r}{2}} \cap \Chi_n = \emptyset  \})
\end{align}
Take the limit on both sides. Since all the right hand side terms go to 0, the probability of the event $ \{ \cost(\pi_{\trans_i\rightarrow \trans_{i+1}}^n) \geq (1+\epsilon)\cost(\pi_{\trans_i\rightarrow \trans_{i+1}}^*) \}$ goes to 0 as well, indicating that this event ceases to happen asymptotically.
\textit{Note that the underlying connection radius has not been changed.}\qed
\end{proof}

\begin{theorem}[Pairwise-asymptotically Optimal Planning Converges in Cost]
\label{thm:pairwiseao}
Given $n_i$ samples in each pairwise motion planning problem $\trans_i$ to $\trans_{i+1}$ over the robust transition sequence $T^+\in\settrans^+$, a path $\Pi^\iter$ is generated after $\iter$ total iterations of the algorithm. For all arbitrarily small $\epsilon_{+}>0$, the following holds:
$$ \lim_{\iter\rightarrow\infty} \prob(\{ \cost(\Pi^\iter) > (1+\epsilon_{+})\cost^*\}) = 0,\quad \forall \epsilon_{+}>0$$
\end{theorem}
\begin{proof}
Following the result of guaranteed convergence for each pairwise motion planning problem
, it needs to be shown that each segment along $T^+$ gets enough attempts ($n_i$) to allow convergence.
Thus a necessary algorithmic condition
is that for each 
motion plan connecting $\trans_i$ to $\trans_{i+1}$, an AO sampling-based roadmap needs $n_i$ samples such that as $\iter\rightarrow\infty,\ n_i\rightarrow\infty,\ \forall \orbit_i$. 

The rest of the proof follows by combining the pairwise results to compose $\Pi^\iter$
Asymptotically it is known:
\begin{align}
\cost(\Pi^\iter) =& \sum_{i=1}^{M-1} \cost( \pi_{\trans_i\rightarrow \trans_{i+1}}^{n_i} )\leq \sum_{i=1}^{M-1} (1+\epsilon_i)\cost(\pi_{\trans_i\rightarrow \trans_{i+1}}^*)\\
\leq& \Big( 1+ \sum_{i=1}^{M-1}\epsilon_i \Big) \cost(T^+)
\leq \Big( 1+ \sum_{i=1}^{M-1}\epsilon_i \Big) (1+\epsilon^+) \cost^*
\leq (1+\epsilon_{+})\cost^*.
\end{align}
Since each of the \textit{epsilon} terms are arbitrarily small by definition, for all small $\epsilon_+>0$, there would be some small enough values of $\epsilon_1,\ldots\epsilon_{M-1},\epsilon^+$ that satisfy the bound.
Similarly the probability is evidently going to 0 since $M$ is independent of $\iter$ and $n_i$. For the estimated values of $\epsilon_1,\ldots\epsilon_{M-1},\epsilon^+$, the probability follows from union bound
\begin{align}
    \lim_{n\rightarrow\infty}& \prob(\{  \cost(\Pi^\iter) > (1+\epsilon^+) \cost^* \}) \\ 
    &\leq \lim_{n\rightarrow\infty} \sum_{i=1}^{M-1} \prob(\{   \cost( \pi_{\trans_i\rightarrow \trans_{i+1}}^{n_i} ) > (1+\epsilon_i)\cost(\pi_{\trans_i\rightarrow \trans_{i+1}}^*\}) = 0
\end{align}

It follows that in a robustly optimal task planning problem the  solution cost from $\alg$ can get arbitrarily close to the optimal solution cost if along a discovered transition sequence $T^+ \in \settrans^+$ it:
\begin{itemize}
        \item solves every pairwise transition connection over an orbit using an asymptotically optimal sampling-based motion planner \textit{with the standard AO radius sufficient for motion planning}
    \item ensures every orbital roadmap ($n_i$) grows infinitely as the total number of high-level iterations ($\iter$) grows to infinity.  
\end{itemize}
This resolves the second part of Eq.~\ref{eq:failure_prob}. Note that the argument holds over $(\startstate, T^+, \goalstate)$.
\qed
\end{proof}

\begin{theorem}[Sampling Sequence of Transitions]
\label{thm:transsample}
A forward search tree $\mathds{T}$
, which selects an orbit per iteration with probability $\Theta(\frac{1}{| \mathds{T} |})$, and uniformly samples an $N_t$ expected number of neighboring transitions every iteration, is guaranteed to expand a sequence of transitions that are $\epsilon^+$-robust.
\end{theorem}
\begin{proof}
Consider orbits to be selected uniformly at random from the working search tree. Let $N_t>0$ be the number of transitions added on each orbit expansion. For each desired pairwise transition $\trans_{i} \rightarrow \trans_{i+1}$ from an orbit $\orbit_i$, robust optimality guarantees the \textit{existence of a positive volume} around $\trans_{i+1}$. The probability of sampling in that volume is a small constant $\varepsilon' = \frac{\mu(\ball{\trans_{i+1}}{\varepsilon})}{\mu_{\cap}} > 0$ which is independent of $N_t$ or $\iter$. Here $\varepsilon$ is some small radius describing the region, and $\mu_{\cap}$ the volume of the submanifold that is being sampled.

This formulation makes the search tree identical in behavior to the \textit{naive random tree} described in previous work~\cite{li2016asymptotically}. Reusing the arguments of \cite{li2016asymptotically}(Theorem 18), by substituting the transition probability with any $\varepsilon'>0$, sampling the correct sequence of transitions from some $T^+$ through the regions defined by Theorem~\ref{thm:robust_task} is guaranteed asymptotically. This resolves the first part of Eq.~\ref{eq:failure_prob}.\qed
\end{proof}

Note that additionally \cite{li2016asymptotically}(Theorem 3) ensures that every orbit will also be expanded infinitely often, guaranteeing all $n_i\rightarrow \infty$ in each orbit when every expansion contributes to a positive expected number ($N_m$) of samples to $n_i$. In our model, $N_t$ simply needs to be positive, and can be a constant by the same argument, to ensure coverage of the boundary submanifold.

\begin{theorem}[Asymptotically Optimal Task Planning]
When applied to a robust task planning problem, solvable by a
finite number of mode transitions, $\alg$ is AO if
\begin{enumerate}
    \item a forward search tree over transition configurations is sampled uniformly at random on the transition submanifold
    \item the number of samples in each orbit and number of neighboring transitions increases asymptotically
    \item the roadmap contained in an orbit uses a connection radius of
    $$ r_n(\orbit) \geq \text{AO motion planning radius in each orbit }\orbit \text{~\cite{Karaman2011Sampling-based-,janson2015fast}} $$
\end{enumerate}
\end{theorem}
\begin{proof}
This follows immediately from Theorem~\ref{thm:pairwiseao} and \ref{thm:transsample}.\qed

\end{proof}

\section{Model for Approaching Boundaries}
\label{sec:boundaries}

The previous sections have relied on the assumption that the boundaries of modes are
smooth (Fig.~\ref{fig:orbital_connection}).
This section shows that a relaxing this assumption to allow start/goal points to lie on a non-smooth boundary
can still result in AO under certain conditions.
The following discussion drops the $ \mode $
superscript notation and refers to $ \cfree \subset \cspace $ as the subspace of
obstacle free configurations (of a single mode).

The typical analysis framework for sampling-based motion planning focuses on
tiling the interior of $ \cfree $ with hyperballs over solution paths.
For smooth boundaries $ \partial\cfree $ it is
possible to tile solution paths with balls that touch the boundary; but, any
irregularity on the boundary will violate this condition. This is readily demonstrated in Fig.~\ref{fig:toyproblem}.

\subsection{Cone Condition}

To argue results for cases where the boundary is not smooth everywhere, this
work borrows certain topological tools for non-smooth boundaries. The proposed
framework still makes some assumptions in terms of the underlying space.

\begin{definition}[Cone]
A ``q-cornet'' \cite{cholaquidis2014poincare} $ \cone{q}{\mathbf{v}}{b}{\phi} $ is the
intersection of a convex cone with apex at $ q $, and a hyperball $ \ball{q}{b}
$ of radius $ b $. The cone is symmetric about vector axis $\mathbf{v}$ and the
``opening'' of the cone is parameterized by $ \phi = \frac{
\mu(\cone{q}{\mathbf{v}}{b}{\phi}) } { \mu(\ball{q}{b})} \in (0,\frac{1}{2}]$.
\end{definition}

With a slight abuse of notation, this work refers to a ``q-cornet'' as a
``cone'', similar to the underlying literature \cite{cholaquidis2014poincare}.

\begin{assumption}[Cone Condition]
For every point $ q\in\partial\cfree $ there exist values $b>0, \phi>0$ and a
vector $\mathbf{v}$ so that there is a cone $ \cone{q}{\mathbf{v}}{b}{\phi} \in \closedfree$.
\label{ass:cone_condition}
\end{assumption}

Note that $\cfree$ automatically satisfies the (Poincar\'e) cone condition
\cite{cholaquidis2014poincare} in its interior since the underlying topology of
the interior contains hyperballs at any configuration.
This assumption is violated in pathological regions, such
as degenerate narrow passages. The following proposition shows that
the cones introduced by Assumption \ref{ass:cone_condition} have a sufficient
intersection with the interior of $\cfree$ to allow sampling
processes to work.

\begin{proposition}[Intersection of Cone and Free Interior]
Given $ q \in \partial\cfree $ and its associated cone $\cone{q}{\mathbf{v}}{b}{\phi}$
from the cone condition, there exists a point $q^{\prime}$ and a small enough
radius $b^{\prime}$ so that $\ball{q^{\prime}}{b^{\prime}} \subset
\cone{q}{\mathbf{v}}{b}{\phi} \cap \cfree,$ i.e., there is a hyperball at the
intersection of the cone and the interior of the free configuration space.
\label{prop:intersection}
\end{proposition}

\begin{proof}
Since $ \phi>0 $, $\exists\ \cone{q}{\mathbf{v}}{b}{\phi}\subset\closedfree$ so that
$\mu(\cone{q}{\mathbf{v}}{b}{\phi}) > 0$. This implies $ \mu(\cone{q}{\mathbf{v}}{b}{\phi} \cap
\cfree) > 0 $, since $ \mu(\partial\cfree) = 0 $. Given the underlying topology
of the space, the positive measure intersection $\cone{q}{\mathbf{v}}{b}{\phi} \cap
\cfree$ can contain a small enough hyperball $\ball{q'}{b'} \subset
\{\cone{q}{\mathbf{v}}{b}{\phi} \cap \cfree\}$ for $b'>0$.
\end{proof}

Let $\vartheta_{q}$ be the supremum radius $b'$ of a hyperball at the
intersection of the cone $\cone{q}{\mathbf{v}}{b}{\phi}$ and $\cfree$. This maximum
radius can also be defined for the start $\qstart$ and goal $\qtarget$ query
points as $\vartheta_{\qstart}$ and $\vartheta_{\qtarget}$, respectively.

\begin{figure}[t]
	\centering
	\includegraphics[width=0.8\textwidth]{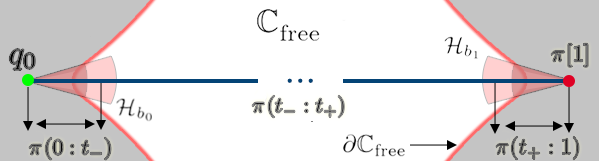}

	\caption
	{
	The figure describes the cone condition for boundary paths showing
	cones at the boundary points, and an intersection with the interior of
	the space.
	}
	\label{fig:conecondition}\vspace{-0.2in}
\end{figure}

\subsection{Robust Clearance for Boundary Paths}

The solution paths for the problems considered by this work must connect points
that lie on the boundary $\partial\cfree$. Such paths will be referred to as
``boundary paths'', as shown in Figure \ref{fig:conecondition}.
\edit{The remainder of this section formally defines boundary paths and extends the notion of $ \delta $-clearance to such paths. Finally, it shows that a motion planning problem that meets this extended notion of clearance for start and goal points can be solved using \ao\ roadmaps.}

\begin{definition}[Boundary paths]
For a boundary path $\pi$ it is true that $\pi[0] \in \partial\cfree $ or
$\pi[1] \in \partial\cfree $.
\end{definition}

Typically, certain clearance properties need to be satisfied for solution paths
in order for sampling-based planners to be able to discover them.
Consider a sequence of subspaces
-- parameterized by a decreasing $\delta > 0$, which approach the entire $\cfree $.

\begin{definition}[$\delta$-interior space]
Given some $ \delta>0 $, the $ \delta $-interior space $ \deltaint \subset
\cfree$ consists of all configurations at least $ \delta $ distance away from $
\partial\cfree$.
\end{definition}

The benefit of the cone condition and resulting proposition is that boundary
paths can be decomposed into three segments: (a) one that passes through the
cone defined at $\pi[0] = \qstart$, (b) a second segment that transitions into
the interior of $ \cfree $, and (c) a third segment that connects to the cone
defined at $\pi[1] = \qtarget$. Given this idea, the notion of ``strong $
\delta $-clearance for boundary paths'' is introduced.

\begin{definition} [Strong $ \delta $-clearance for Boundary Paths]
A boundary path $ \pi $ satisfies ``strong $ \delta $-clearance for boundary
paths'' for some $ \delta>0 $, if there are path parametrizations $0 \leq t_{-}
< t_{+} \leq 1$ for the path, so that:

\begin{itemize}[label=$-$]

\item the subset of the path $\pi(0:t_{-})$ from $\qstart =
	\pi[0]\in\partial\cfree $ to $ \pi[t_-] \in \deltaint$ lies entirely in
	some $ \cone{\qstart}{\mathbf{v}_0}{\hat{b_0}}{\phi_0} \subset \closedfree$ for
	$\hat{b_0}, \phi_0 > 0$.

\item the subset of the path $\pi(t_{-}:t_{+})$ lies in the $ \deltaint$

\item the subset of the path $\pi(t_{+}:1)$ from $ \pi[t_+] \in \deltaint$ to
	$\qtarget = \pi[1]\in\partial\cfree $ lies entirely in some $
	\cone{\qtarget}{\mathbf{v}_1}{\hat{b_1}}{\phi_1} \subset \closedfree,$ for
	$\hat{b_1},\phi_1 > 0$.

\end{itemize}
\end{definition}

The construction of strongly $ \delta $-clear boundary paths is feasible by
Proposition \ref{prop:intersection} for some $ 0<\delta\leq
\frac{1}{2}\min(\vartheta_{\qstart},\vartheta_{\qtarget}) $, when $ \qstart $
and $ \qtarget $ are connected through $ \cfree $. In general, any range of $
\delta $ where such a construction is possible can be considered. A view of
such a path at one of the ends is shown in Figure \ref{fig:pi_epsilon} (right).

Note that the optimal boundary path $\pi^* $ for a motion planning problem may
approach arbitrarily close to obstacle boundaries and hence violate strong $
\delta$-clearance conditions for boundary paths. In order to model regions of
space in the vicinity of $ \pi^* $, which contain ``near-optimal'' paths with a
cost that gets arbitrarily close to $ \cost(\pi^*) $, the notion of $
\delta_\epsilon $-clear convergence 
from prior work
\cite{janson2015fast} is adopted here.
$\pi^*$ has to exhibit some ``weak clearance'' and allow
the existence of a sequence of strong $\delta$-clearance boundary paths that
converge to it.

\begin{figure}[t]
	\centering
		\includegraphics[height=1.2in]{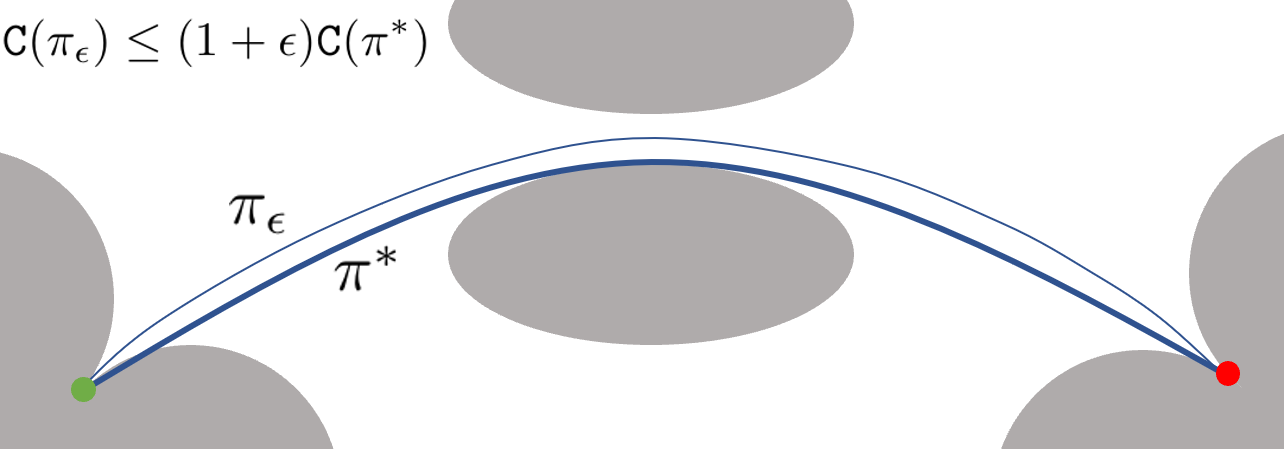}
	\includegraphics[height=1.2in]{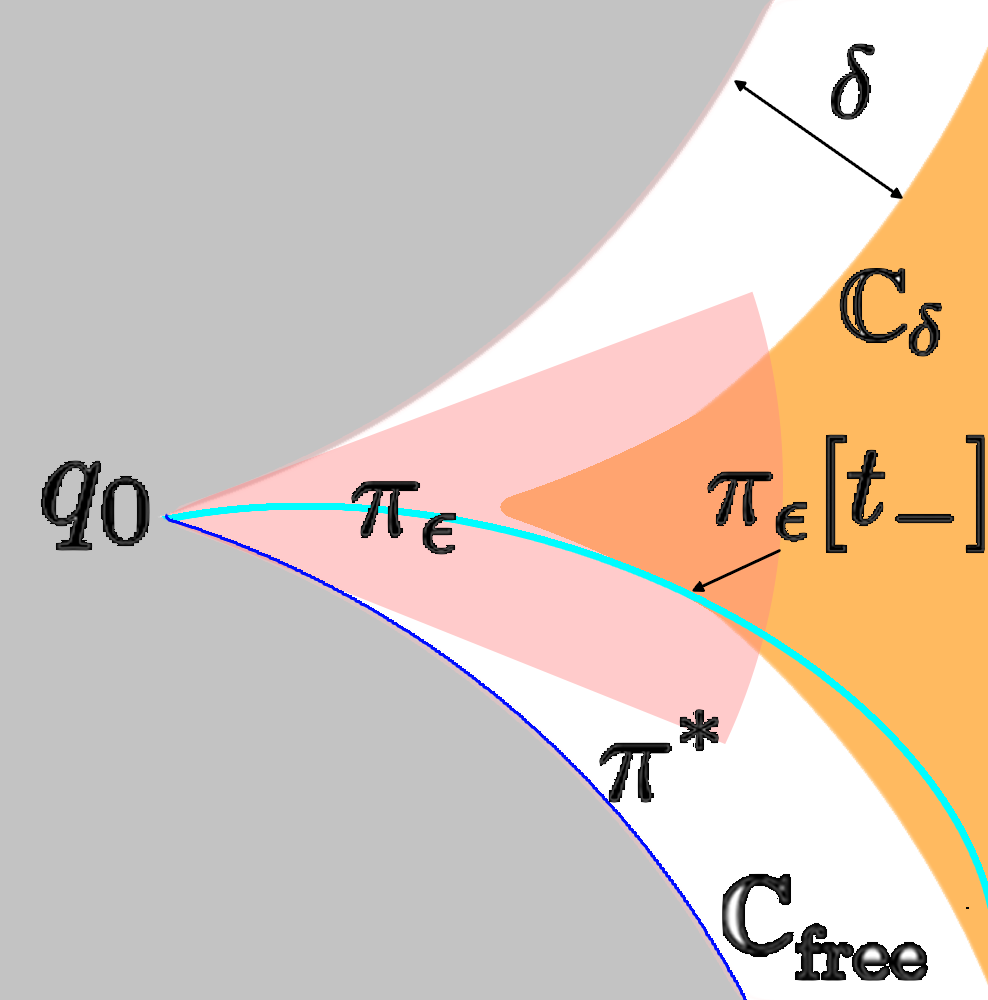}
			\caption { (\textit{Left:}) Strong $ \delta_\epsilon$-clearance Convergence, and (\textit{Right:}) Strong $ \delta $ convergence of some $ \pi_\epsilon $ w.r.t. a $ \pi^*
	$ that traces the boundary.  }
	\label{fig:pi_epsilon}
\vspace{-0.2in}
\end{figure}

\begin{definition}[Strong $ \delta_\epsilon$-clearance Convergence]
A motion planning problem exhibits ``strong $\delta_\epsilon$-clearance
convergence'' for boundary paths, if for all small $ \epsilon>0 $, there exists
some range of clearance values $\delta \in (0, \delta_\epsilon]$, such that a
``strong $ \delta $-clear boundary path'' $ \pi_\epsilon $ with
$\pi_\epsilon[0] = \pi^*[0]$ and $\pi_\epsilon[1] = \pi^*[1]$, has its cost
bounded relative to the cost of the optimal path $ \pi^* $ as follows: $ \cost(\pi_\epsilon) \leq (1+{\epsilon}) \cost(\pi^*) $.

\label{def:clearconv}
\end{definition}

\begin{assumption}
Assume the motion planning problem under inspection exhibits ``strong
$\delta_\epsilon$-clearance convergence'' (as shown in
Fig.~\ref{fig:pi_epsilon})\label{ass:strongdelta}.
\end{assumption}

\begin{theorem}[AO Connection Radius Is Sufficient for Non-smooth Boundaries]
In the case where $ \qstart,\qtarget \in \partial\cfree$, the connection radius that guarantees asymptotic optimality of the interior of the space is sufficient for connecting $\qstart$ to $\qtarget$.
\label{thm:boundary_conn}
\end{theorem}
\begin{proof}
The arguments are identical to the proof of Theorem~\ref{thm:pairwiserobust}. The probability of failing to connect these boundary points can be split into the probability of failing to connect each boundary point to the interior of the space, and the probability of failing to connect the two interior points with bounded error. 
The only difference occurs in Eq.~\ref{eq:connect_bound0}-\ref{eq:connect_bound}, where the ball is replaced by the cone such that the volume now decreases by a constant fraction $\phi$. Recall that $\Chi_n$ is a set of $n$ samples and $\oneball$ denotes the volume of a ball of radius one.
\vspace{-0.2in}
\begin{align}
    \prob(\{ \cone{\trans_i}{\mathbf{v}}{\phi}{r_n} \cap \Chi_n = \emptyset  \}) =& \Big(  1 - \frac{\mu(\cone{\trans_i}{\mathbf{v}}{\phi}{r_n})}{\mu(\cfree)} \Big)^n \\
    = \Big( 1 - \frac{\phi\oneball r_n^d}{\mu(\cfree)} \Big)^n \leq& \ e^{-  \frac{\phi\oneball}{\mu(\cfree)} n r_n^d }\\
    \implies \lim_{n\rightarrow\infty} \prob(\{ \cone{\trans_i}{\mathbf{v}}{\phi}{r_n} \cap \Chi_n = \emptyset  \}) &= 0 ,\text{when}\ \ \lim_{n\rightarrow\infty}n r_n^d \rightarrow \infty 
    \label{eq:cone_bound}
\end{align}
This bound is still readily satisfied by the connection radius argued in asymptotically optimal roadmap-based methods(\cite{Karaman2011Sampling-based-,janson2015fast}). It follows from the other arguments that motion planning inside an orbit, and task planning across a robustly optimal sequence of transitions both converge to the optimal cost asymptotically.\qed
\end{proof}

An instance of such a motion planning problem for increasingly narrow \textit{cones} is demonstrated in Fig.~\ref{fig:coneplanning}. Note that in the context of task planning the non-smoothness of the boundaries necessitate a stronger assumption about the \textit{transition sampler} in task planning. The transition sampler needs to be aware of the precise submanifold to sample, which might no longer be the dimensionality of the boundary, since the apex of the cones might lie on a lower dimensional space. This is reasonable in practice. For instance, a grasp sampler for a parallel gripper will sample transitions that constrain the alignment of the object between the fingers.

\begin{figure}
\vspace{-0.2in}
    \centering
    \includegraphics[width=0.25\textwidth]{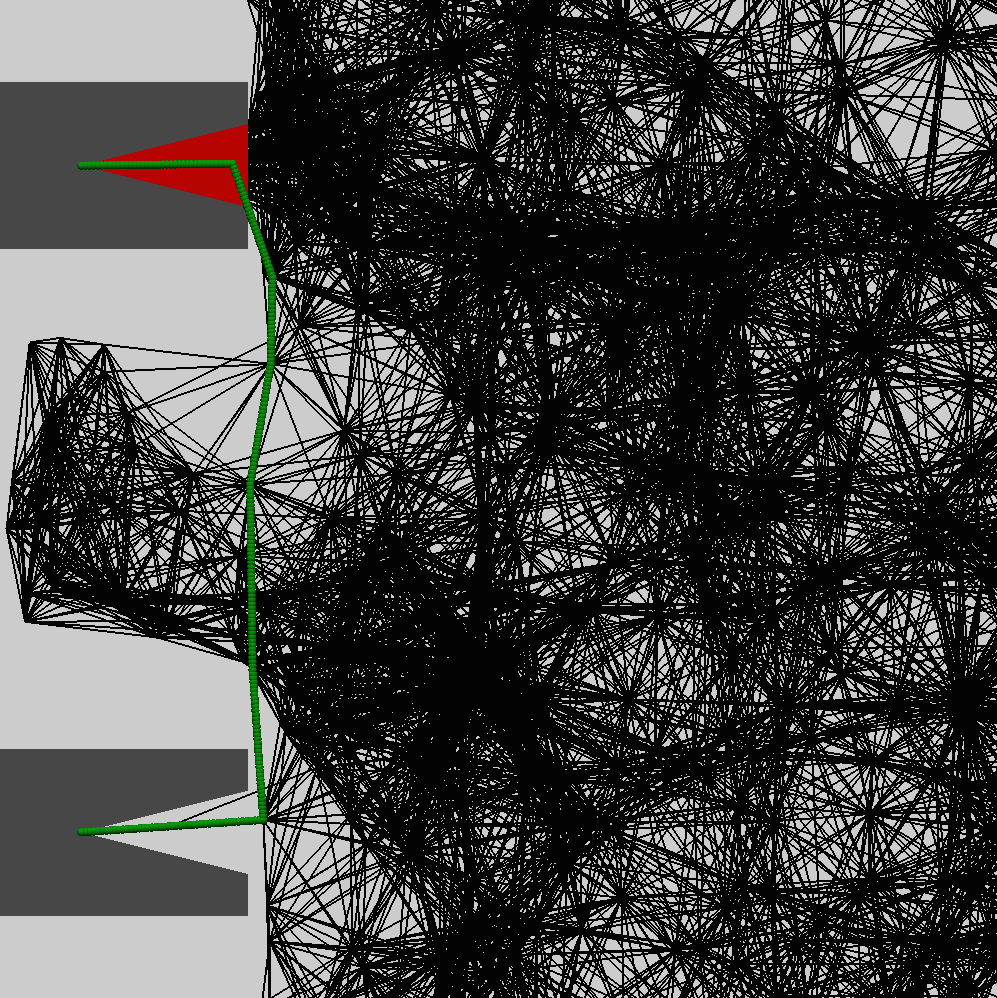}
    \includegraphics[width=0.25\textwidth]{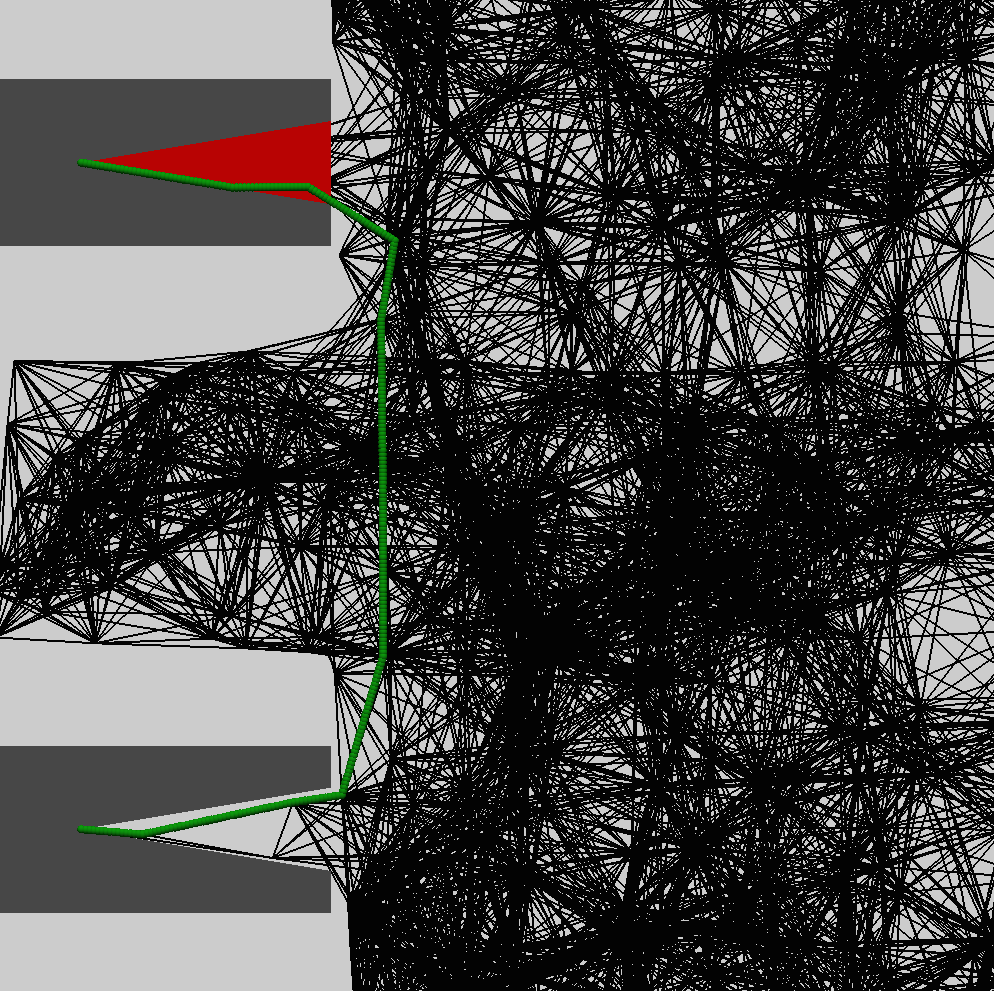}
    \includegraphics[width=0.25\textwidth]{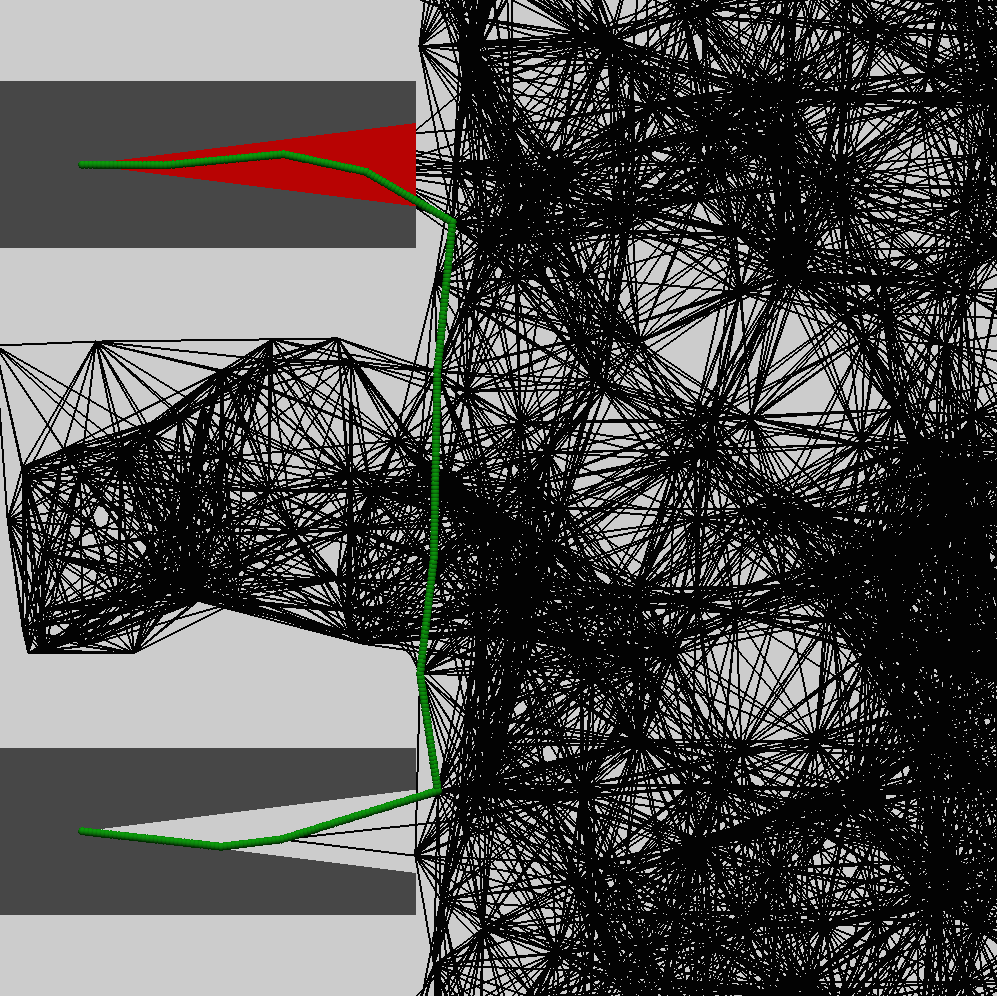}
\vspace{-0.1in}
    \caption{Instances of motion planning between two non-smooth boundary points in 2D, using the radius from \prmstar, for a red triangular robot that can only translate. The start is shown, and the green arc shows the  path traced by the robot's apex to reach the lower triangular cavity.}
    \label{fig:coneplanning}
    \vspace{-0.4in}
\end{figure} 

\section{Discussion}
\label{sec:discussion}
This work aims to highlight useful analysis tools for understanding the
properties of sampling-based \tamp\ planners in order to achieve AO.
In particular, this paper argues that given reasonable assumptions about the \tamp planning problem \edit{an algorithm that can guarantee sufficient sampling of mode transitions as well as orbit interiors will retain the AO properties of the underlying sampling-based motion planner.}
\edit{Another contribution is the relaxation of the smooth boundary assumption widely applied by AO motion and \tamp planners.
This work shows that it suffices that such boundary points are countably finite along the path and permit hypercones that open into the interior of the space. Moving away from smooth boundaries is reassuring as smoothness is hard to justify in practical problems, especially in manipulation problems with contact. It is also an issue for many motion planning queries, such as docking at a charging station.}

\edit{
In term of scalability, and practical performance, the algorithmic structure presented in this work is rather general, and if implemented naively would prove inefficient. Devising effective admissible heuristics in \tamp\ is critical for the quick discovery of high-quality solutions. It is also of interest to study whether the domain of transition sequences permit smoothing operations that are typically used in motion planning. The description of the homotopic properties in the multi-modal space can prove useful tools in improving \tamp\ solutions.
The study of convergence rates and the inspection of finite time properties for AO planners is an important consideration. 
The design of practically efficient \tamp planners for realistic problems remains an active research area. 
}

{
\bibliographystyle{spmpsci}

}

\end{document}